\documentclass{egpubl}
\usepackage{eg2019}

\ConferenceSubmission   

 \electronicVersion 

\ifpdf \usepackage[pdftex]{graphicx} \pdfcompresslevel=9
\else \usepackage[dvips]{graphicx} \fi

\PrintedOrElectronic

\usepackage{t1enc,dfadobe}

\usepackage{egweblnk}
\usepackage{cite}

\usepackage{amsmath}
\usepackage{amsfonts}
\usepackage{pgfplots}
\usepackage{pgf,tikz}
\usepackage{subcaption}
\usepackage{rotating}
\usetikzlibrary{calc,patterns,angles,quotes,positioning}
\usepackage{xcolor}
\usepackage[utf8]{inputenc}
\usepackage{overpic}

\newtheorem{theorem}{Theorem}
\newtheorem{proposition}[theorem]{Proposition}

\title[Divergence-Free Shape Interpolation and Correspondence]
      {Divergence-Free Shape Interpolation and Correspondence}

\author[M. Eisenberger, Z. Lähner, D. Cremers]{Marvin Eisenberger \quad \quad Zorah Lähner \quad \quad Daniel Cremers}

\begin{document}

\teaser{
\includegraphics[width=.32\linewidth]{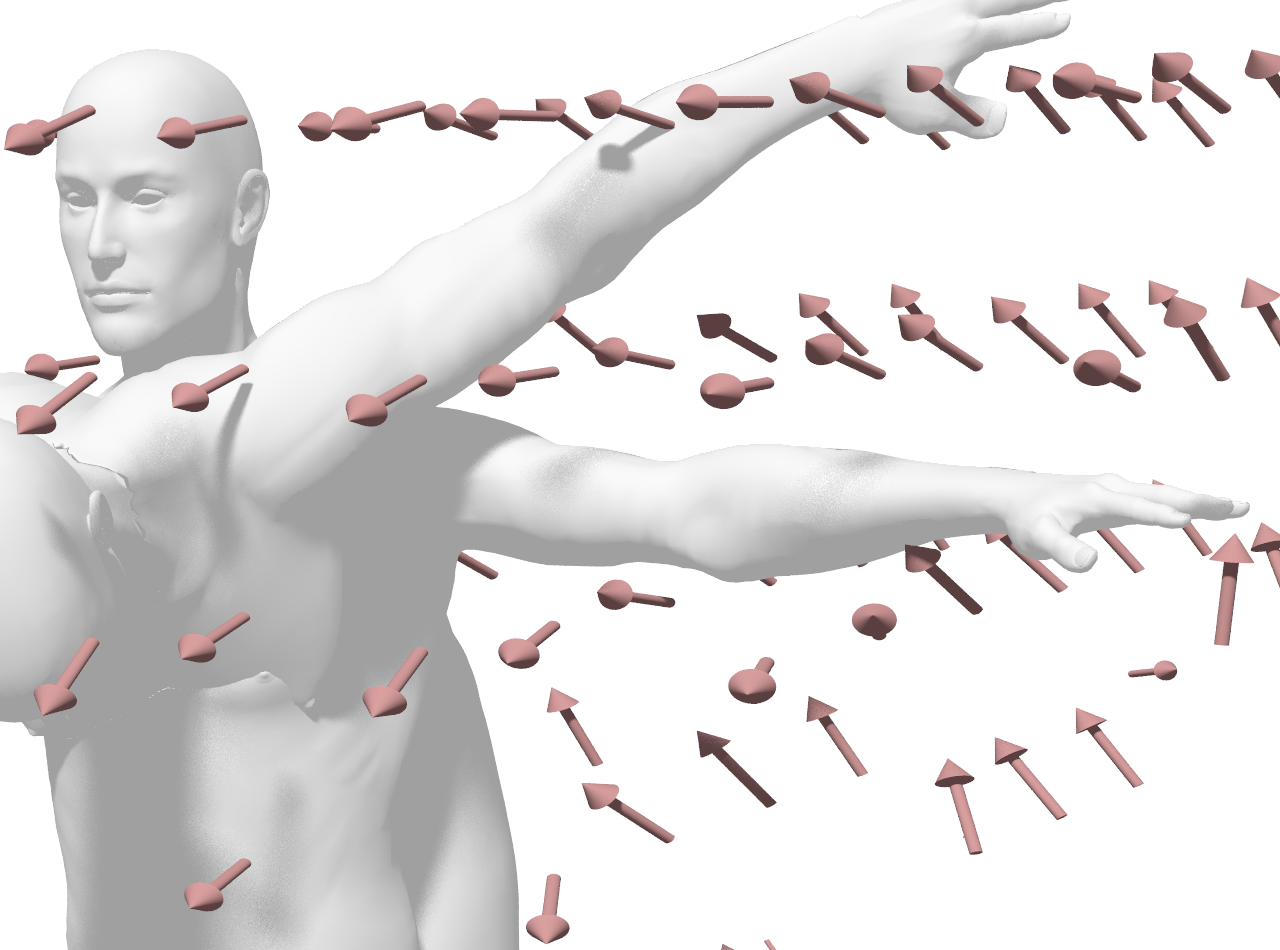}
\includegraphics[width=.64\linewidth]{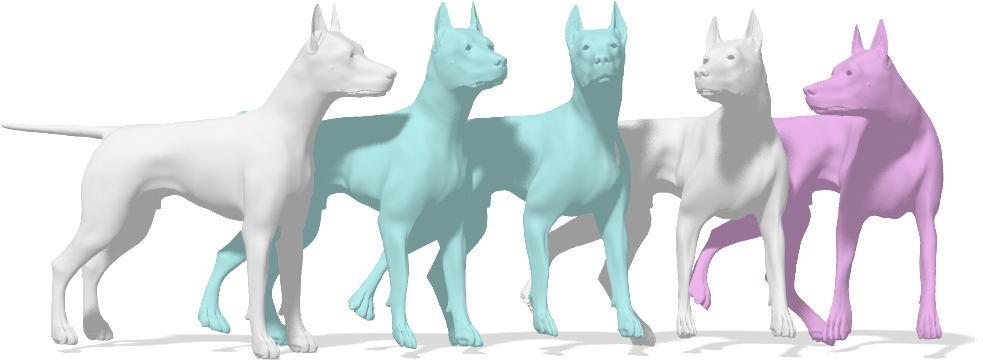}
 \centering
  \caption{Given two input shapes we propose to compute
    a volume preserving deformation field by imposing a
    zero-divergence constraint. The resulting method
    alternates between optimizing the deformation field
    and calculating a correspondence for a small subset
    of vertices. (Left) Example of a deformation field
    in 3D. (Right) Our approach generates a continuous
    family of intermediate shapes along with a highly
    accurate correspondence. The input shapes are shown in white,
    intermediate shapes in blue and one extrapolation is shown in pink. }
\label{fig:teaser}
}

\maketitle

\begin{abstract}
  We present a novel method to model and calculate
  deformation fields between shapes embedded in $\mathbb{R}^D$.
  Our framework combines naturally interpolating
  the two input shapes and calculating correspondences
  at the same time. The key idea is to compute a divergence-free
  deformation field represented in a coarse-to-fine basis using
  the Karhunen-Lo\'{e}ve expansion. The advantages are that there
  is no need to discretize the embedding space and the deformation
  is volume preserving. Furthermore, the optimization is done
  on downsampled versions of the shapes but the morphing can be applied to
  any resolution without a heavy increase in complexity. We show
  results for shape correspondence, registration, inter- and extrapolation
  on the TOSCA and FAUST Scan data sets.


\end{abstract}

\section{Challenges in Shape Analysis}\label{sec:introduction}

Handling non-rigidly deformed versions of a 3D shape is at the
heart of numerous problems in computer vision and graphics
ranging from shape comparison, information and style transfer
to the automatic generation of new but meaningful
shapes. In contrast to rigid shape registration which has six degrees of freedom, the problem of finding non-rigid deformations is rather ambiguous and the complexity of the optimization grows
quadratically with the number of vertices in the input shapes.

While many of these problems are intrinsically related, most existing methods
address them independently and do not generalize to a wider range of tasks.
For example, having the same texture on both input shapes can help to
find correspondences between these, while knowing the correspondence makes it
easy to transfer the texture. Doing joint optimization can help to improve
the performance in both tasks. Our approach combines shape registration,
correspondence and interpolation -- bringing different viewpoints to the same question of how to model non-rigid shape deformations.

\emph{Shape registration} aims at finding a transformation of a
shape $\mathcal{X}$ that aligns its surface with another
shape $\mathcal{Y}$ as closely as possible. This is trivial when
the \emph{correspondence} $\pi: \mathcal{X} \to \mathcal{Y}$
is already given, because then we can prescribe a deformation which
aligns each point $x \in \mathcal{X}$ with its match $\pi(x) =
y \in \mathcal{Y}$.  Vice versa, when two surfaces are aligned
a correspondence can easily be found by searching for the
nearest neighbor of each point in the embedding space.
The	\emph{shape interpolation} problem addresses the task of finding a sequence of
intermediate shapes between $\mathcal{X}$ and $\mathcal{Y}$
such that $\mathcal{X}$ is transformed into $\mathcal{Y}$ in a
natural and continuous way. Most methods tackling this
task assume they are given a perfect correspondence between
the initial shapes but the existence of a smooth transition is
actually a requirement for a good correspondence.

Many state-of-the-art correspondence methods can find fairly accurate
matchings but they often rely on pointwise information like descriptor values
or pairwise distances. While these are straightforward to compute, information about
continuity is not encoded or increases the complexity of the optimization
heavily. Consequently, outliers due to intrinsic symmetries or different sampling
prevent a continuous morphing.
In this work we will introduce a representation that can formulate both problems and makes a joint optimization possible.

\section{Related Work}\label{sec:relatedwork}

\subsection{Deformation Fields}

Deformation fields have a long history in image
registration. Ashburner and colleagues made use of deformation
fields for autonomous shape morphing \cite{ashburner2007fast}.
They consider temporally constant deformation fields
offering limited flexibility to capture more complex
deformations. Solving for a space and time dependent
deformation field is a highly underdetermined problem. A
remedy for this issue is provided by the geodesic shooting
approach advocated by
\cite{miller2006geodesic} which only estimates the initial
velocity field for each pixel and then how the
velocity has to propagate in the image domain in order to
preserve the kinetic energy and the momentum of the whole
system. Further improvements of this framework were proposed
in subsequent work, including a Gauss-Newton approach
\cite{ashburner2011diffeomorphic} and a particularly efficient
adjoint calculation \cite{vialard2012diffeomorphic}.

Closely related to our work is \cite{von2006vector} in which the authors also model volume preserving shape deformations using divergence-free vector fields. Here, deformation fields are constructed from hand crafted templates which are meant to be used as interactive shape transformation tools.

It is often beneficial to have a probabilistic interpretation
of deformation fields. This yields a comprehensive description
with explicit schemata to impose uniformity on vector fields
as well as a sound theoretical foundation. Such a model for
image registration and 2D shape registration with a Gaussian
process modeling of the correspondence mapping is proposed in
\cite{albrecht2008statisticaldeformationprior}. Further work
\cite{luethi2016gaussianprocessmorphablemodels,doelz2017efficientcomputation}
specified how one can extend this approach to
Gaussian processes on the surface of a three dimensional
shape.  The authors in \cite{bregler2000recoveringnonrigid},
\cite{torresani2008nonrigidsfm},
\cite{albrecht2008statisticaldeformationprior} and
\cite{paladini2009factorization} also model non-rigid
transformations using a PCA type representation of permitted
motions. Analogously,
\cite{myronenko2010pointsetregistration} and
\cite{ma2014robust} pursue a reproducing kernel Hilbert space
approach to model the vector field interpolation. However, for
all these references the respective vector fields are not
defined on the whole embedding space surrounding the shapes but
rather only at the elements of the considered point clouds and
they do not admit an interpretation as a deformation field.

Another classical approach to shape deformation is based on a rotation invariant representation of triangle meshes \cite{lipman2005linear}. In \cite{zhang2008deformation} it is then presented how this deformation model can be used to compute a sparse set of correspondences.

\subsection{Shape Registration and Matching}

Much work has been done in the direction of shape registration
and matching and we would like to point the interested reader
to in-depth surveys of these topics for an overview
\cite{vankaick11correspsurvey,salvi2007rangereview,tam2013pointcloudsurvey}. Here
we will focus on work that is directly related to our
approach.

A popular line of work in shape matching is based on spectral
decomposition of the surface Laplace-Beltrami operator
\cite{Dubrovina2010MatchingSB}.  This is popular because it
reduces the dimensionality of the problem from the number of
vertices to the number of basis functions chosen
\cite{ovsjanikov2012functional}. Nevertheless, extracting the
correspondence from the low dimensional representation is
still a complex problem and often retrieved solutions are
noisy or hard to compute \cite{rodola2015cpd}.  We also use a
spectral approach but, instead of a basis for functions on the
surface, we represent deformation fields in the embedding
space using the eigenfunctions of the standard Laplacian.

Methods based on Multi-Dimensional Scaling find
correspondences by reembedding and then aligning shapes in a
(possibly smaller) embedding space where the complexity is
reduced \cite{bronstein2006generalized,aflalo2016spectral}.
\cite{chen2015robust} calculate a robust non-rigid registration
based on Markov random fields but can not retrieve a continuous
deformation.
In \cite{myronenko2010pointsetregistration} and
\cite{ma2014robust} the authors address the non-rigid
registration problem by modeling one point cloud as a Gaussian
mixture model. Moreover, they also determine the
correspondences and point mappings in an alternating manner
using a expectation maximization algorithm. This
work is strongly related to our framework. Like
our approach they directly model the correspondence mapping of
all points as a Gaussian process type mapping. There also exist extensions of this method which additionally include descriptor values \cite{ma2016non,ma2017feature}.

\subsection{Shape Interpolation}

Although registration methods often compute a deformation
between shapes, the focus is not on producing realistic
intermediate shapes. Most methods realistically interpolating
between shapes assume to be given a full
correspondence in advance. \cite{kilian2007geometric} interprets the solution to interpolation as geodesics in the pointwise shape space. \cite{wirth2011shapespace} and
\cite{heeren2016shellsplines} model a space of shells with
metrics induced by physical deformation energies in which
geodesics or splines represent natural interpolations between
shells. \cite{vontycowicz2015realtime} make real-time
interpolation on a set of preprocessed, given shapes with
arbitrary resolution possible.

\cite{alexa2000asrigidas} only needs a handful of correct
correspondences to find a volume preserving deformation
between two shapes but they require both to be segmented in
compatible simplicial complexes.
In \cite{xu2005poisson} a dense deformation field similar to ours
is calculated but the method depends on a consistent triangulation
of the inputs.
Other directions include taking user input to guide deformations in the right direction \cite{vaxman2015moebiustrans}
or rely on a known or learned model to generate new shapes \cite{Gao2017DataDrivenSI}.

\section{Contribution}\label{sec:contribution}

In the following, we will introduce a mathematical framework
which allows to jointly tackle the problems of shape
interpolation/extrapolation, shape registration and
correspondence estimation. Our method solely operates on two given 3D point clouds and in particular requires no connectivity information like a mesh. We propose to
estimate a smooth and volume preserving 3D deformation field
prescribing a plausible interpolation of these input shapes. More
specifically, we solve an initial value problem for determining the shape deformation. This framework allows us to incorporate physical
assumptions about the deformation field. We
suggest to impose volume preservation by enforcing zero
divergence.  More specifically, we represent the deformation
field as the curl of a potential function and propose a
natural coarse-to-fine basis representation of these potential
functions.  The initial value problem is then integrated by a
Runge-Kutta scheme. We use an expectation maximization approach to
simultaneously determine a subset of the unknown point-to-point
correspondences and the optimal deformation field parameters.
The objective is aligning two shapes with a preferably uniform
deformation field. We
demonstrate that the proposed framework can be used to create
plausible shape interpolations and extrapolations in numerous experiments. Moreover,
it provides a shape correspondence which compares to
state-of-the-art correspondence methods.

\begin{figure*}
	\centering
	\begin{subfigure}[b]{0.45\linewidth}
		\includegraphics[width=\linewidth]{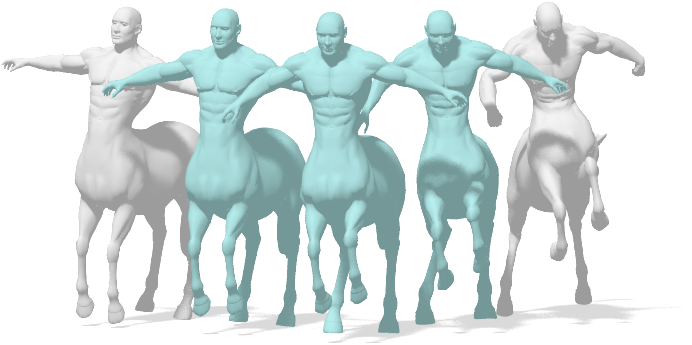}
		\caption{Centaur.}
		\label{fig:ipol_rearing}
	\end{subfigure}
	\begin{subfigure}[b]{0.48\linewidth}
		\includegraphics[width=\linewidth]{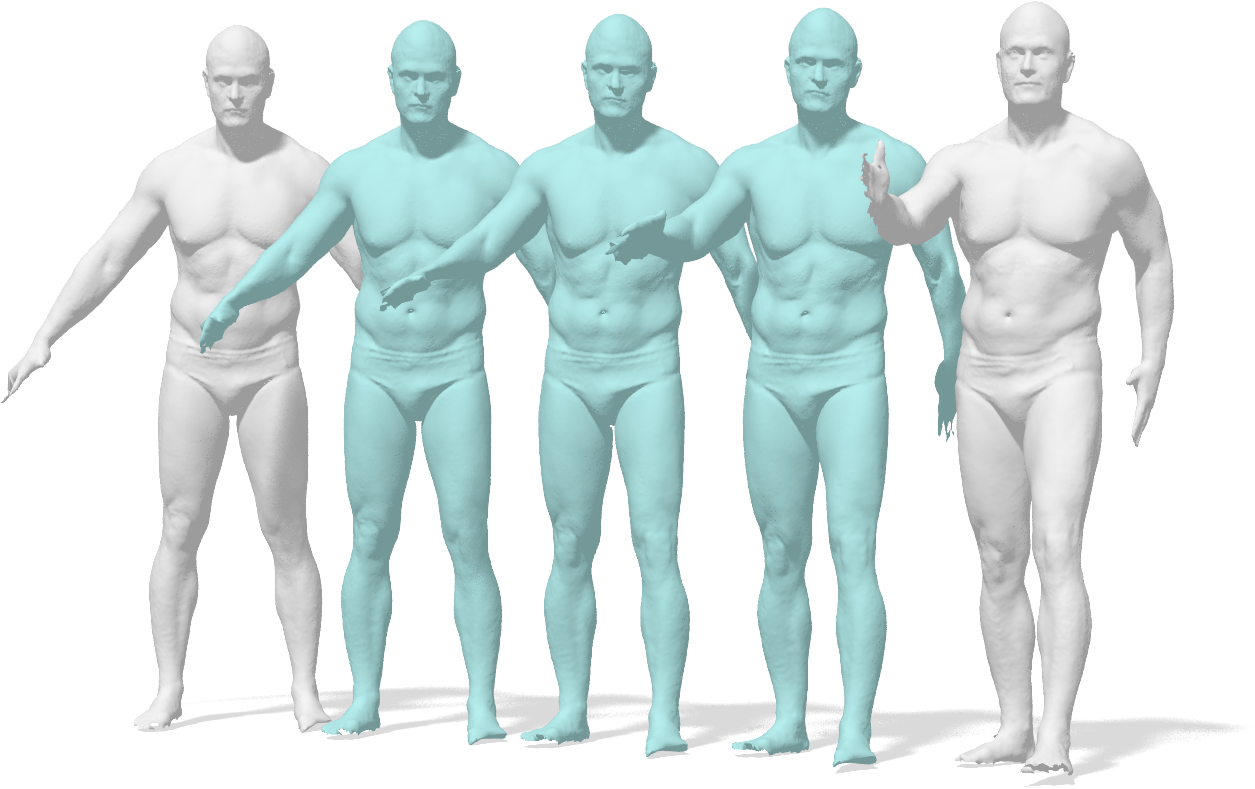}
		\caption{Human.}
		\label{fig:ipol_faust}
	\end{subfigure}
  \begin{subfigure}[b]{0.95\linewidth}
		\includegraphics[width=0.9\linewidth]{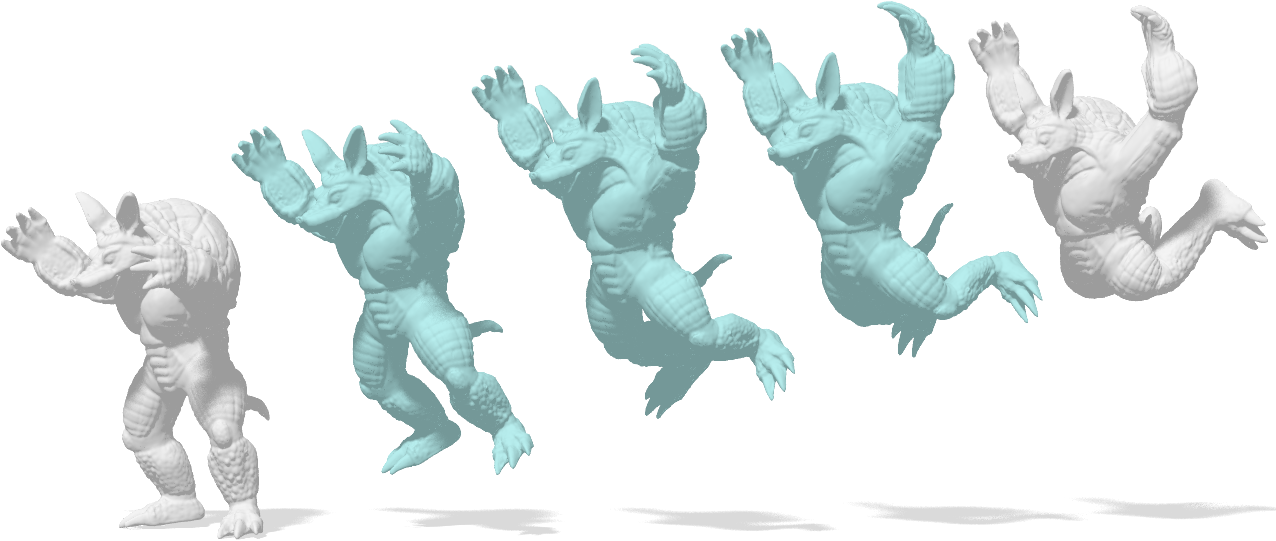}
		\caption{Armadillo.}
		\label{fig:ipol_armadillo}
	\end{subfigure}
	\caption{Three examples of shapes that are morphed into one another according to the initial value problem of Eq.~\eqref{eq_initalvalueproblem}. The centaur (a) and the human (b) are from the TOSCA \cite{bronstein2008numerical} and FAUST \cite{Bogo:CVPR:2014} dataset respectively. The armadillo (c) is from the AIM@SHAPE shape repository \cite{aim-shape}.  (b) is a scan of a real person and very high resolution ($214k$ vertices). The source and target shape are shown in white and the interpolations at times $t=0.25, 0.5, 0.75$ in blue. The translation is not part of our deformation and was only introduced for clarity in the figures.
	}
	\label{fig:interpol}
\end{figure*}

\section{Problem Formulation}\label{sec:formulation}

This section gives an introduction into the problem we want to solve and the mathematical background we use in later sections.

\subsection{Deformation field shape correspondences}

Consider two discrete sets of points $\mathcal{X}=\{x_1,\dots,x_N\}\subset\Omega$ and $\mathcal{Y}=\{y_1,\dots,y_M\}\subset\Omega$ contained in a compact domain $\Omega\subset\mathbb{R}^D$. The points $x_n$ and $y_m$ are assumed to be uniformly sampled from the surface of two similar $D$-dimensional shapes. The shape registration problem now addresses the task of aligning the point clouds $\mathcal{X}$ and $\mathcal{Y}$ in a meaningful manner, such that similar regions of the two shapes are matched onto each other. In particular we are looking for a morphing $f:\Omega\to\Omega$, such that the mapped points $f(x_n)$ fit to the shape $\mathcal{Y}$.

Our approach chooses these mappings $f$ in such a way that they imitate plausible transformations in the real world. For this purpose we make some natural assumptions about the trajectories of the transformed points $x_n$. For once we require the points of our shape to move smoothly over time. We would also like points in a certain neighborhood to shift in a uniform manner. This assumption yields to some extent that the determined correspondences are continuous. Moreover, the volume of shifted objects should remain the same.
We therefore assume that every point $x_n\in\mathcal{X}$ moves according to the following initial value problem:

\begin{equation}
\label{eq_initalvalueproblem}
\begin{cases}
\dot{x}(t)=v(x(t)).\\
x(0)=x_n.
\end{cases}
\end{equation}

In this context $v:\Omega\to\mathbb{R}^D$ is some fixed deformation field moving the point $x_n$ over time. In order to ensure a uniform movement of the samples the vector field $v$ has to be sufficiently smooth. We will even assume that $v\in C^\infty(\Omega,\mathbb{R}^D)$, which yields the following well-known result:
\begin{proposition}
	For a deformation field $v\in C^\infty(\Omega,\mathbb{R}^D)$, the initial value problem \eqref{eq_initalvalueproblem} admits a solution on the compact domain $\Omega$ and this solution is unique. Furthermore it is infinitely many times continuously differentiable $x\in C^\infty([0,1],\mathbb{R}^D)$
\end{proposition}
\begin{proof}
	This follows directly from the theorem of Picard-Lindelöf \cite[Theorem 2.5]{teschl2012ordinary} and \cite[Lemma 2.3]{teschl2012ordinary}.
\end{proof}

As a consequence we can choose the correspondence mapping $f$ to be the solution operator of \eqref{eq_initalvalueproblem} evaluated at an arbitrary time $t$. For convenience we choose $t=1$:
\begin{equation}
\label{eq_deformationoperator}
f:\begin{cases}
\Omega\to\mathbb{R}^D.\\
x_n\mapsto x(1).
\end{cases}
\end{equation}

The advantage of this framework is that it provides an intuitive notion of deformation when analyzing shape correspondences. Therefore, one does not only get a matching of two shapes but also all intermediate states constituting the underlying transformation. Those are typically more meaningful than merely performing linear interpolation between the initial and the final position of each point $x_n$. Especially when looking at nearly isometric shapes having a continuous correspondence and a natural deformation are inherently connected and solving for both simultaneously improves the results considerably.

\subsection{Divergence free deformations}

Another convenient quality of the presented framework is that it enables us to incorporate assumptions about the deformation field into our model. One reasonable restriction arises from the mathematical investigation of fluid dynamics \cite{chorin1993fluid}, namely the restriction to divergence free velocity fields:
\begin{equation}
\label{eq_divergencezero}
\nabla\cdot v=0.
\end{equation}

It is well known that this local property confines the vector field in such a way, that it yields volume conservation over time for any considered part $U\subset\Omega$ of the shape:

\begin{proposition}
	\label{proposition:volumepreservation}
	Consider an open set $U\subset\Omega$. Let now $U(t)$ be the set of solutions of \eqref{eq_initalvalueproblem}, if each point in $U$ is shifted individually. If we assume that the deformation field $v$ is divergence-free \eqref{eq_divergencezero}, then the volume of $U(t)$ is constant over time:
	\begin{equation}
	\frac{\mathrm{d}}{\mathrm{d}t}\int_{U(t)}\mathrm{d}x=0.
	\end{equation}
\end{proposition}
\begin{proof}
	This statement is a particular case of the result in \cite[Lemma 8.8]{teschl2012ordinary}.
\end{proof}

Helmholtz's theorem \cite{rutherford1962vectorstensors} implies that any sufficiently smooth vector field on the compact domain $\Omega$ can be decomposed into the sum of a curl free and a divergence free component. It furthermore provides us with an explicit construction of the divergence free component of any velocity field:
\begin{equation}
\label{eq_vectorfieldfrompotential}
v:=\nabla\times\Phi.
\end{equation}

In this context $\Phi:\Omega\to\mathbb{R}^D$ is a $C^\infty$ potential field. This vector field $\Phi$ arises from a construction in the Helmholtz decomposition and it typically depends on the considered vector field $v$. Unfortunately constructing $\Phi$ from $v$ is not straightforward and it is in general not computationally feasible. Therefore we allow the potential $\Phi$ to be an arbitrary vector field. Then we can define $v$ to be its curl in analogy to \eqref{eq_vectorfieldfrompotential}. Either way we get a divergence free vector field $v$ due to the following basic property of the curl operator:
\begin{equation}
\nabla\cdot(\nabla\times\Phi)=0.
\end{equation}

In the case of $D=3$ spatial dimensions the construction of $v$ in \eqref{eq_vectorfieldfrompotential} admits the following form:
\begin{equation}
\label{eq_vectorfieldfrompotential3D}
v=\begin{pmatrix}
\partial_2\Phi_3-\partial_3\Phi_2 \\ \partial_3\Phi_1-\partial_1\Phi_3 \\ \partial_1\Phi_2-\partial_2\Phi_1
\end{pmatrix}=\begin{pmatrix}
0 \\ \partial_3\Phi_1 \\ -\partial_2\Phi_1
\end{pmatrix}+\begin{pmatrix}
-\partial_3\Phi_2 \\ 0 \\ \partial_1\Phi_2
\end{pmatrix}+\begin{pmatrix}
\partial_2\Phi_3 \\ -\partial_1\Phi_3 \\ 0
\end{pmatrix}.
\end{equation}

\section{Method}\label{sec:method}

Regarding relevant applications we will mainly restrict ourselves to the case of $D=3$. However, extensions to higher dimensions or the 2D case are straightforward.

\subsection{Spatial representation}

We need to describe the velocity fields $v$ in a more tangible manner such that we obtain a computationally feasible method.
The problem is that there are infinitely many choices for functions $\Phi\in C^\infty(\Omega,\mathbb{R}^D)$. The most straightforward approach is choosing a discretization of the embedding space, e.g. with a voxel grid. The potential and deformation fields can then be defined by assigning a three-dimensional vector to every voxel. The problem with this approach is that it has cubic complexity which becomes costly very fast. Furthermore, we do not get a spatially continuous deformation and loose volume conservation and other desirable properties as a consequence. In the following we will introduce a finite, linear basis $\{v_1,...,v_K\}$ for any velocity field on $\Omega$ and derive a formulation to restrict it to only span smooth, divergence-free fields. The number of basis function can be adjusted for either speed or expressiveness.

The eigenfunctions of the Laplace-Beltrami operator are often used in shape analysis because of their useful properties like invariance to non-rigid deformations, smoothness and natural ordering. We use a similar basis for smooth, divergence-free vector fields in $\mathbb{R}^D$. Without loss of generality the considered domain is assumed to be a $D$-dimensional cube $\Omega:=[0,1]^D$ and we translate and scale any shape to generously fit inside.
We start with the eigenfunctions $\{\phi_1,\phi_2,...\}$ of the standard Laplacian $\Delta$ on $\Omega$:
\begin{equation}
\Delta\phi_k=\lambda_k^\Delta\phi_k.
\end{equation}

This basis of eigenfunctions $\{\phi_1,\phi_2,...\}$ is ordered such that the eigenvalues $0\geq\lambda_1^\Delta\geq\lambda_2^\Delta\geq...$ are descending. Furthermore, we require the Laplacian to admit Dirichlet boundary conditions $\Phi|_{\partial\Omega}=0$. The $\phi_k$ can be determined analytically and they are exactly the $\sin$ elements of the Fourier basis:
\begin{equation}
\label{eq_fourierbasis}
\{\phi_1,\phi_2,...\}=\biggl\{\prod_{d=1}^D\frac{1}{2}\sin(\cdot\pi j_d)\bigg|(j_1,...,j_D)\in\mathbb{N}^D\biggr\}.
\end{equation}

These functions $\phi_k$ now form an orthonormal basis wrt. the $\|\cdot\|_{L^2(\Omega)}$ norm. The eigenvalue $\lambda_k^\Delta$ of $\Delta$ corresponding to the eigenfunction $\phi_k$ is the following:
\begin{equation}
\lambda_k^\Delta:=-\pi^2\sum_{d=1}^{D}j_d^2.
\end{equation}

We can now map the potential basis to the velocity field to obtain a feasible description $v$. For this purpose we directly insert the basis elements $\phi_k$ from \eqref{eq_fourierbasis} into the places of the components $\Phi_1,...,\Phi_D$ of $\Phi$ in \eqref{eq_vectorfieldfrompotential}. Due to the linearity of the curl operator $\nabla\times\cdot$ this can be done for every entry of $\Phi$ at a time, see \eqref{eq_vectorfieldfrompotential3D} for $D=3$. Overall, we obtain a basis $\{v_1,v_2,...\}$ of the velocity field $v$. In the 3D case it is explicitly defined as:
\begin{equation}
\label{eq_basisV}
\{v_1,v_2,...\}=\bigcup_{k=1}^\infty\biggl\{\begin{pmatrix}
0 \\ \partial_3\phi_k \\ -\partial_2\phi_k
\end{pmatrix},\begin{pmatrix}
-\partial_3\phi_k \\ 0 \\ \partial_1\phi_k
\end{pmatrix},\begin{pmatrix}
\partial_2\phi_k \\ -\partial_1\phi_k \\ 0
\end{pmatrix}\biggr\}.
\end{equation}

In this context the basis elements $v_k$ are again sorted according to the eigenvalues $\lambda_k^\Delta$ of the corresponding $\phi_k$ in descending order. Note that there are in general multiple basis functions $v_k$ for each eigenvalue $\lambda_k^\Delta$.
This yields a feasible description of the velocity field by computing $v$ as a linear combination of the first $K$ basis elements:
\begin{equation}
\label{eq_trunckarhunenloeve}
v(x)=\sum_{k=1}^{K}v_k(x)a_k.
\end{equation}

The coefficients $a_k$ can be defined as random variables with a Gaussian prior distribution $a_k\sim\mathcal{N}(0,\lambda_k)$. This assumption yields a probability distribution over all admissible deformation fields $v$ which we will use later to obtain the optimization energy function. The weights $\lambda_k$ are constructed from the eigenvalues $\lambda_k^\Delta$ in the following manner:
\begin{equation}
\label{eq_lambdak}
\lambda_k:=\bigl(-\lambda_k^\Delta\bigr)^{-\frac{D}{2}}=\biggl(\pi^2\sum_{d=1}^{D}j_d^2\biggr)^{-\frac{D}{2}}.
\end{equation}

Intuitively this kind of weighting promotes a damping of the high frequency components of $v$ and therefore yields a uniform vector field, because the coefficients $a_k$ are sampled with a smaller variance $\lambda_k$. This also justifies the truncation in \eqref{eq_trunckarhunenloeve}, because the weights corresponding to high frequency basis elements are insignificantly small anyway. The mathematical background of the sampling approach in \eqref{eq_trunckarhunenloeve} is provided by the Karhunen-Loève expansion \cite[Ch. 11]{sullivan2015UC} which is an extension of the principal component analysis (PCA) for more general vector spaces.

The Dirichlet boundary conditions automatically guarantee that there is no flow out of the domain $\Omega$, in the sense that the components of $v_k$ are orthogonal to the outer normals at the boundary $\partial\Omega$. This can be easily verified for the case $D=3$ by computing the basis elements \eqref{eq_basisV} inserting \eqref{eq_fourierbasis}, but we refrain from proving this property here. It also becomes obvious by looking at instances of the velocity basis functions in Figure~\ref{fig:basisFunctions}.

\newsavebox{\matcap}
\savebox{\matcap}{$\begin{pmatrix}\partial_2\phi_{\tilde{k}}\\-\partial_1\phi_{\tilde{k}}\\0\end{pmatrix}$}

\begin{figure}
	\centering
	\includegraphics[trim={40 0 0 0},clip,width=\linewidth]{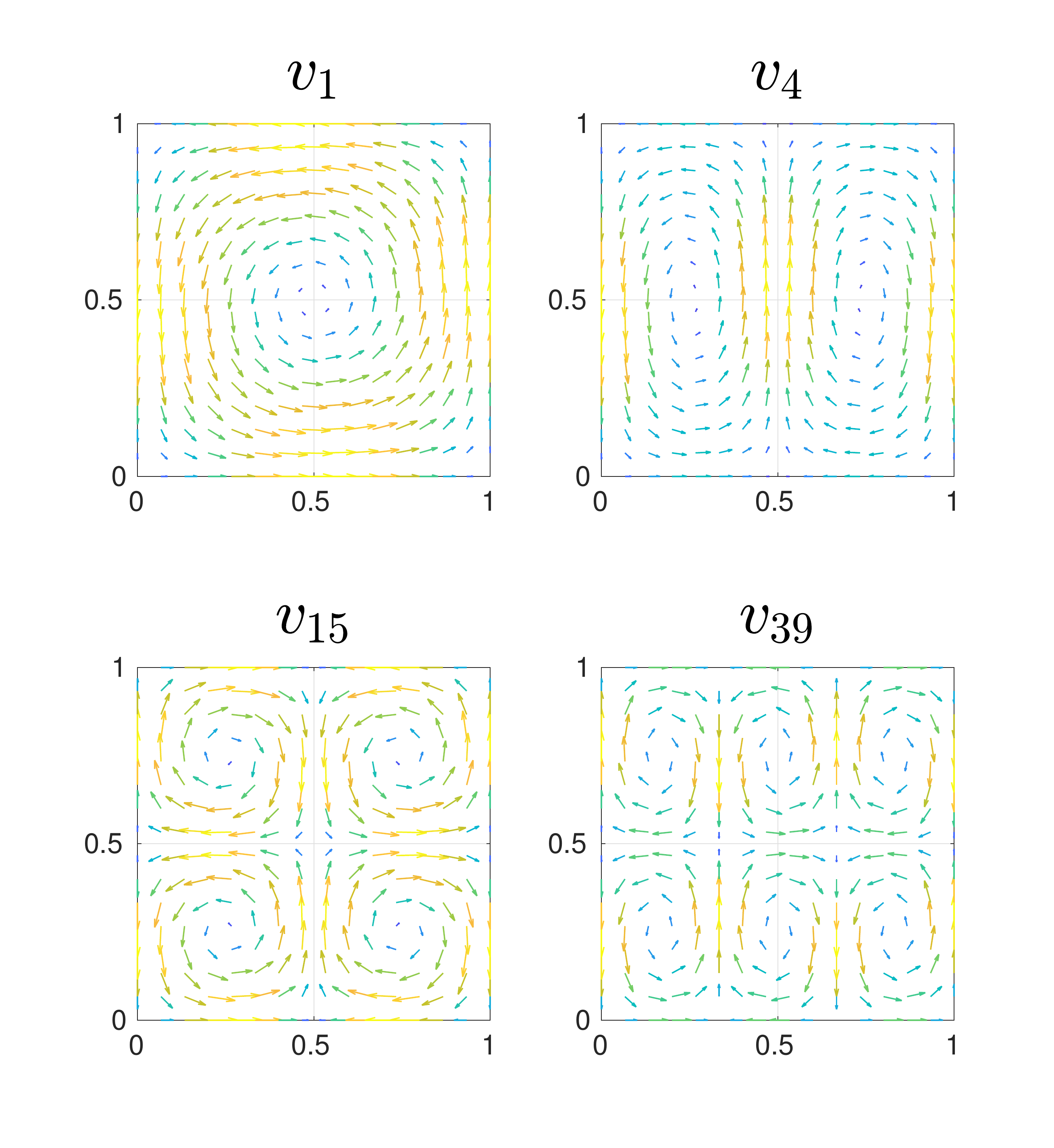}
	\caption{Cross section of some deformation field basis functions $v_k:\Omega\to\mathbb{R}^3$ at $x_3=0.5$. Notice the low frequency structures for low $k$ and
		increasing frequencies with higher indices.
	}
	\label{fig:basisFunctions}
\end{figure}

Our representation of the velocity field has several beneficial properties. First of all, it can be evaluated at any point in the domain $\Omega$ and we do not need to introduce a spatial discretization. We will use this in our experiments to make our method applicable to very high resolution shapes. Secondly, the particular weighting of the summands and the truncation in \eqref{eq_trunckarhunenloeve} induces an inbuilt low-pass filtering which produces smooth deformations because it favors the low frequency basis elements.

Classical shape registration methods like the iterative closest point (ICP) algorithm \cite{besl1992method} determine a rigid transformation between two shapes. Interestingly, our approach can be interpreted as a direct extension of these methods. The deformation fields corresponding to translations and rotations are divergence free and uniform, and therefore, contained in our framework by design. This becomes even clearer when the basis function $v_1$ in Figure \ref{fig:basisFunctions} is examined. It is fairly similar to a rotation around the $x_3$ axis. This especially holds near the center of the domain $\Omega$ and deteriorates at its boundary $\partial\Omega$. Actually it is fairly straightforward to verify, that this equivalence holds up to first order around the center. The basis elements $v_2$ and $v_3$ express rotations around the $x_2$ and $x_1$ axis in a similar way. The considerations in this subsection raise the question whether a cubic domain $\Omega$ is the best choice for our purposes. Following the work in \cite{zhao2007unitCircle}, \cite{zhao2008unitCircle} we could pursue our approach in a spherical domain. This would lead to more complex basis functions $v_k$ but the first three eigenfunctions would span the space of rotations without undesirable artifacts at the boundaries of the domain. Although this would be a nice theoretical property we refrain from using these basis functions here due their complex structure, especially because this does not lead to a quantifiable improvement of our method.

\subsection{Temporal discretization}

In order to evaluate the correspondence mapping $f$ in \eqref{eq_deformationoperator} we have to solve the initial value problem \eqref{eq_initalvalueproblem} by a numerical integration scheme. The simplest choice in this context is the explicit Euler method. However, we decided to use a second order Runge-Kutta method \cite[Ch. 9]{griffiths2010numericalmethods}, because it has a significantly higher accuracy and therefore enables us to choose a coarser discretization. We subdivide the time space in an equidistant grid with $T\in\mathbb{N}$ intervals and set the step size $h=\frac{1}{T}$. This yields the following explicit iteration scheme:
\begin{equation}
\label{eq_rungekutta}
\begin{cases}
x_n^{(0)}:=x_n.\\
x_n^{(t+1)}:=x_n^{(t)}+hv\biggl(x_n^{(t)}+\frac{h}{2}v\bigl(x_n^{(t)}\bigr)\biggr).\\
f_n:=x_n^{(T)}.
\end{cases}
\end{equation}

We typically choose $T\in\{1,...,100\}$ in our experiments. In general, we have to make a trade off between runtime and accuracy when selecting a proper number of steps $T$. We commonly also get meaningful transformations if we choose $T$ to be small but we might loose some key properties of our framework like the volume preservation. This effect is illustrated in Figure \ref{fig:rotatedBat} for the 2D shape of a bat transformed by a 90 degree rotation around the center. Note that the deformation field corresponding to this transformation is actually not contained in our framework due to our choice of domain and boundary conditions, see discussion in the previous subsection. If for this setup we now choose too few time steps $T$, the shape shifts outward and the area expands. On the other hand, this effect becomes insignificantly small if we choose $T\geq 10$.

\begin{figure}
	\centering
  \includegraphics[width=.45\linewidth]{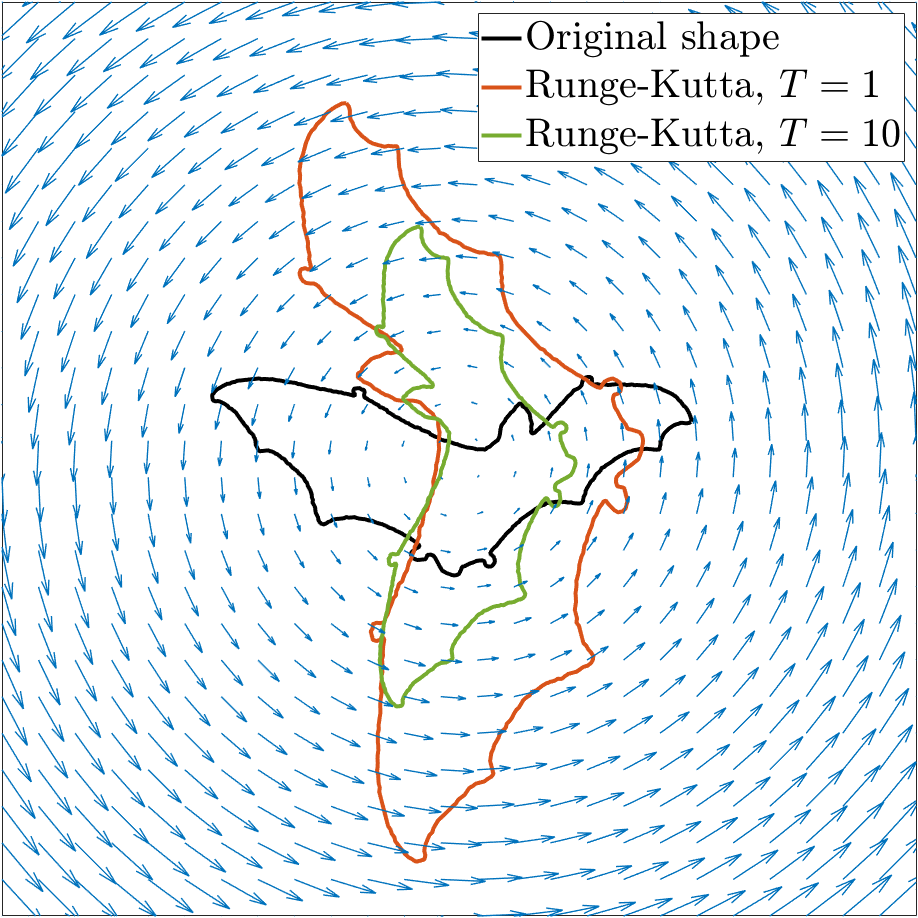}
%
%
\definecolor{mycolor1}{rgb}{0.00000,0.44700,0.74100}%
\begin{tikzpicture}

\begin{axis}[%
width=.34\linewidth,
at={(2.6in,1.284in)},
scale only axis,
xmin=0,
xmax=100,
xlabel={\footnotesize $T$},
xlabel style={at={(0.5,0.05)}},
ymode=log,
ymin=1e-06,
ymax=10,
yminorticks=true,
ylabel={\tiny Relative Area Expansion},
ylabel style={at={(0.16,0.48)}},
yticklabels={\tiny{$10^{-6}$},\tiny{$10^{-3}$},\tiny{$10^0$\ \ }},
axis background/.style={fill=white},
legend style={legend cell align=left,align=left,draw=white!15!black},
xmajorgrids,
ymajorgrids,
]
\addplot [color=mycolor1,solid,line width=2.0pt,forget plot]
  table[row sep=crcr]{%
1	1.52201704740628\\
2	0.199301099256236\\
3	0.0574368661302349\\
4	0.0239944434283398\\
5	0.0122355842910901\\
6	0.00706709572757771\\
7	0.00444581507647887\\
8	0.00297655855662781\\
9	0.00208975324928436\\
10	0.00152305991176752\\
11	0.00114410849494529\\
12	0.000881152565904963\\
13	0.000692992167807806\\
14	0.000554813800355196\\
15	0.000451062933415037\\
16	0.000371650923494753\\
17	0.000309838989493127\\
18	0.000261009025013903\\
19	0.000221924047022495\\
20	0.00019026932494692\\
21	0.000164359806397604\\
22	0.000142948993748176\\
23	0.000125101348676623\\
24	0.000110105421421812\\
25	9.74136456648427e-05\\
26	8.65999325785649e-05\\
27	7.73293560625881e-05\\
28	6.93361847552401e-05\\
29	6.24077617805846e-05\\
30	5.63725376662236e-05\\
31	5.10910903835468e-05\\
32	4.64493191778518e-05\\
33	4.23532376561631e-05\\
34	3.87249554873708e-05\\
35	3.54995519729981e-05\\
36	3.26226248233613e-05\\
37	3.00483544176267e-05\\
38	2.77379647296519e-05\\
39	2.56584917645626e-05\\
40	2.37817920785009e-05\\
41	2.20837399828705e-05\\
42	2.05435739789006e-05\\
43	1.91433619282337e-05\\
44	1.7867561228667e-05\\
45	1.67026554268515e-05\\
46	1.56368526388892e-05\\
47	1.46598342011491e-05\\
48	1.37625443305083e-05\\
49	1.29370134181965e-05\\
50	1.21762090260476e-05\\
51	1.14739097927336e-05\\
52	1.08245983626826e-05\\
53	1.02233701691312e-05\\
54	9.66585547794671e-06\\
55	9.14815256292743e-06\\
56	8.66677025521083e-06\\
57	8.21857841320611e-06\\
58	7.80076510653369e-06\\
59	7.41079950467424e-06\\
60	7.04639963147129e-06\\
61	6.70550427874645e-06\\
62	6.38624848540631e-06\\
63	6.08694207924512e-06\\
64	5.80605086076418e-06\\
65	5.54218006413857e-06\\
66	5.29405979599657e-06\\
67	5.060532183183e-06\\
68	4.84054001158635e-06\\
69	4.63311666403253e-06\\
70	4.43737718976184e-06\\
71	4.25251036722758e-06\\
72	4.07777163588621e-06\\
73	3.91247679188826e-06\\
74	3.7559963552686e-06\\
75	3.60775053349408e-06\\
76	3.4672047020496e-06\\
77	3.33386535355125e-06\\
78	3.2072764559833e-06\\
79	3.08701617307779e-06\\
80	2.97269390884254e-06\\
81	2.86394764425328e-06\\
82	2.76044152452993e-06\\
83	2.66186368131146e-06\\
84	2.56792425621758e-06\\
85	2.47835361093569e-06\\
86	2.39290069923373e-06\\
87	2.31133159022431e-06\\
88	2.23342811984334e-06\\
89	2.1589866675066e-06\\
90	2.08781703499412e-06\\
91	2.01974142787362e-06\\
92	1.95459351977696e-06\\
93	1.89221759982238e-06\\
94	1.83246779028376e-06\\
95	1.77520732926678e-06\\
96	1.72030791432156e-06\\
97	1.66764909595405e-06\\
98	1.61711772466819e-06\\
99	1.56860743863943e-06\\
100	1.52201819414063e-06\\
};
\end{axis}
\end{tikzpicture}%
	\caption{Area expansion with different step sizes using the Runge-Kutta integration.
  Left: Rotation around $90$ degrees on a a bat shape of the MPEG-7 dataset \cite{mpeg-7} (black).
  If executed in one step ($T=1$) the shape expands (red) whereas for ten steps $T=10$ the area of the
  interior stays nearly the same (green). Right: Relative area expansion when performing the same deformation with increasing amount of steps $T$.
  }

	\label{fig:rotatedBat}
\end{figure}

\subsection{Optimization}\label{subsec:optimization}

In the previous sections we derived a coherent description of shape morphing using volume preserving deformation fields. We can now use this framework to construct an algorithm that matches two given point clouds $\mathcal{X}$ and $\mathcal{Y}$ while at the same time computing plausible interpolated shapes. In order to do that we need to simultaneously optimize for the deformation fields and the unknown correspondences. We encode the latter using a soft correspondence matrix $W\in[0,1]^{N\times M}$. High values of $W_{nm}\approx 1$ indicate a high correspondence probability for the point pair $(x_n,y_m)$. Moreover, we use the coefficients $a\in\mathbb{R}^K$ to represent the deformation fields $v(x)$ which according to \eqref{eq_trunckarhunenloeve} are completely determined by $a$.

Similar to \cite{myronenko2010pointsetregistration} and \cite{ma2014robust} we address shape registration in a probabilistic manner. We interpret the point cloud $\mathcal{X}$ as a Gaussian mixture model with the means located at the shifted points $f_n=x_n^{(T)}$ and the covariance $\sigma^2 I_D\in\mathbb{R}^{D\times D}$ for some $\sigma>0$. This enables us to simultaneously determine the deformation field coefficients $a\in\mathbb{R}^K$ and the correspondences $W\in[0,1]^{N\times M}$ by applying an expectation maximization approach.

Only using Euclidean distance as the measure of similarity between points suffices to capture rigid deformations but this fails if source or target undergo large \emph{non-rigid} deformations and often leads to incorrect local optimums. Therefore, we incorporate point-wise feature descriptors
in our model. They can account for large scale deformations but also encode information about fine scale structures which steers the optimization towards the right optimum.
Specifically, we use the SHOT descriptor \cite{tombari2010SHOT} with the distance function $\mathrm{d}^{\text{SHOT}}_{nm} = \|\text{SHOT}(x_n) - \text{SHOT}(y_m) \|_2$ on the sets of points $\mathcal{X}$ and $\mathcal{Y}$. The standard GMM formulation only uses the Euclidean distance $\mathrm{d}^{\text{Euclid}}_{nm}=\|y_m-f_n\|_2$ of the shifted point $f_n$ and $y_m$. We want to use a combination of both, therefore we will define the distance of $f_n$ and $y_m$ to be:
\begin{equation}
	\label{eq_distancemetric}
	\mathrm{d}:=\mathrm{d}^{\text{Euclid}}+\frac{\overline{\mathrm{d}}^{\text{Euclid}}}{\overline{\mathrm{d}}^{\text{SHOT}}}\mathrm{d}^{\text{SHOT}}.
\end{equation}

Here, $\overline{\mathrm{d}}\geq 0$ is the mean distance of a metric $\mathrm{d}$ regarding all point pairs in $\mathcal{X}$ and $\mathcal{Y}$ and the factor at $\mathrm{d}^{\text{SHOT}}$ is used to ensure that both metrics have a comparable scaling. Note that $\mathrm{d}$ is a metric on the point clouds $\mathcal{X}$ and $\mathcal{Y}$ as a positive combination of metrics. Using this notion of distance, we can specify how to update the soft correspondences $W$ using the current estimate of the deformation field parameters $a$. This corresponds to the E step of the expectation maximization algorithm:

\begin{equation}
	\label{eq_weightsW}
	W_{nm}:=\frac{\exp\biggl(-\frac{1}{2\sigma^2}\mathrm{d}_{nm}^2\biggr)}{(2\pi\sigma^2)^{\frac{D}{2}}+\sum_{\tilde{n}=1}^{N}\exp\biggl(-\frac{1}{2\sigma^2}\mathrm{d}_{\tilde{n}m}^2\biggr)}.
\end{equation}

In the context of \eqref{eq_trunckarhunenloeve} it was mentioned that the prior of the coefficients $a$ is a Gaussian distribution $a\sim\mathcal{N}(0,L)$, where $L:=\mathrm{diag}(\lambda_1,...,\lambda_K)$. Together with the GMM assumption for the shape $\mathcal{X}$ we can construct the following energy function for the M step:

\begin{equation}
	\label{eq_energyE}
	E(a):=\frac{1}{2}a^TL^{-1}a+\frac{1}{\sigma^2}\sum_{m=1}^{M}\sum_{n=1}^{N}W_{nm}\rho(\|y_m-f_n\|_2).
\end{equation}

The function $\rho:\mathbb{R}\to[0,\infty)$ is the Huber loss \cite{huber1964huberloss}:
\begin{equation}
\rho(r)=\begin{cases}
\frac{1}{2}r^2&|r|\leq r_0.\\r_0|r|-\frac{1}{2}r_0^2&\text{otherwise.}
\end{cases}
\end{equation}

In our evaluations, we choose the outer slope as $r_0:=0.01$. Note that the Huber loss does not directly arise from the standard GMM formulation, but it admits another probabilistic interpretation as an additive mixture of Huber density functions. It also helps to account for outliers and makes the deformation field estimation more robust in general.

In our experiments we apply a Gauss-Newton type approach to minimize the energy in \eqref{eq_energyE}. This results in an iterative method similar to the Levenberg-Marquardt algorithm \cite{levenberg1944leastsquares}. The overall expectation maximization algorithm now alternates between computing the weights $W^{(i)}$ according to \eqref{eq_weightsW} and performing one Gauss-Newton update step to obtain $a^{(i)}$. To initialize the algorithm we usually set the deformation field to be zero $a^{(0)}:=0$.

\section{Experiments}\label{sec:experiments}

We perform experiments for several applications including shape matching and interpolation to show that our method is general
and flexible. Although we handle shapes with up to $200k$ and more vertices, the computation of the deformation field is always done on
a downsampled version with $3000$ vertices and then applied to the full resolution. We use Euclidean farthest point sampling.
When downsampling the subset should include all relevant fine scale structures in order for the deformation field to move these
correctly but we found $3000$ sufficient for all our applications. As a preprocessing step, we align the shapes using PCA and shift
them such that the empirical mean of the point clouds corresponds to the middle of the domain. When averaging over all experiments presented here, our algorithm takes about $370$ seconds to compute all correspondences for one pair of shapes. Due to our a priori downsampling the runtime is pretty consistent and almost independent of the number of vertices. All experiments were performed with MATLAB on a system with an Intel Core i7-3770 CPU clocked at 3.40GHz, 32 GB RAM and a GeForce GTX TITAN X graphics card running a recent Linux distribution. In all our experiments, we solely operate on the raw shape data and in particular do not need any ground truth correspondences.

\subsection{Matching}\label{subsec:expmatching}

We verify our method using the TOSCA dataset \cite{bronstein2008numerical} which contains 76 triangular meshes. The dataset is divided into $8$ classes of humans and animals with several poses each and known intraclass correspondences.

We set the hyperparameters $\sigma^2:=0.01,\,T:=20$ and choose $K=3000$ basis functions for the deformation field. Because $W^{(i)}$ only contains $3000$ correspondences, we perform a nearest-neighbor search with the metric \eqref{eq_distancemetric} to obtain a dense mapping. The evaluation is done with the Princeton benchmark protocol \cite{kim11}. Given the ground-truth match $(x, y^*) \in \mathcal{X} \times \mathcal{Y}$, the error of the calculated match $(x, y)$ is given by the geodesic distance between $y$ and $y^*$ normalized by the diameter of $\mathcal{Y}$.
\begin{equation*}
\epsilon(x) = \frac{ \mathrm{d}^{\text{Geo}}_\mathcal{Y}(y, y^*) }{ \text{diam}(\mathcal{Y}) }
\end{equation*}

We plot cumulative curves showing the percentages of matches that are below an increasing threshold. As zero is the value for ground-truth matches, the ideal curve would be constant at $100$. See Figure~\ref{fig:errorTOSCA} for our results and Figure~\ref{fig:exampleTOSCA} for an example matching.

\begin{figure}
	\centering
	\input{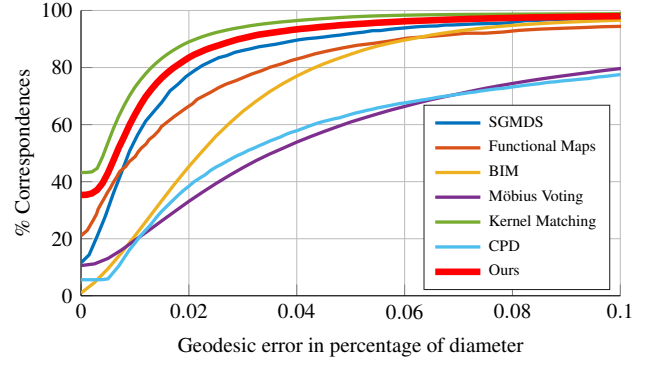}
	\caption{Quantitative evaluation using the Princeton benchmark protocol on the TOSCA dataset \cite{bronstein2008numerical}. We compare with Spectral Generalized Multi-Dimensional Scaling (SGMDS) \cite{aflalo2016spectral}, Functional Maps \cite{ovsjanikov2012functional}, Blended Intrinsic Maps (BIM) \cite{kim11}, M\"obius Voting \cite{lipman2009mobius}, Coherent Point Drift (CPD) \cite{myronenko2010pointsetregistration} and Kernel Matching \cite{kernel17}. Only our method and CPD model an extrinsic morphing of the shapes in the embedding space. }
	\label{fig:errorTOSCA}
\end{figure}

\begin{figure}
	\centering
	\includegraphics[width=.48\linewidth]{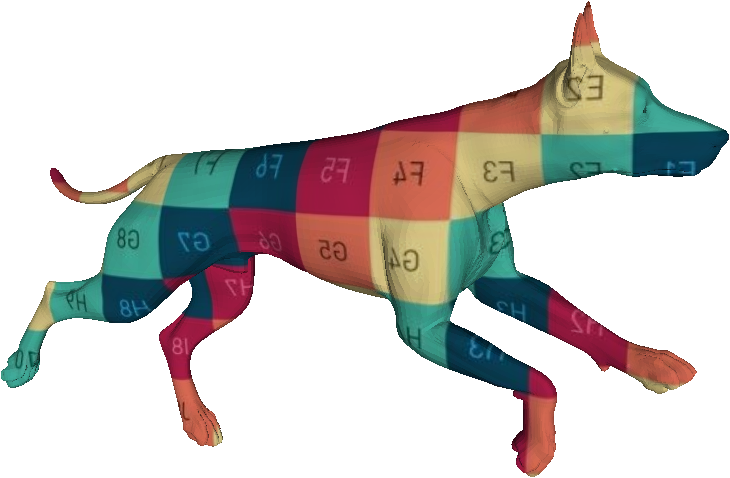}
	\includegraphics[width=.42\linewidth]{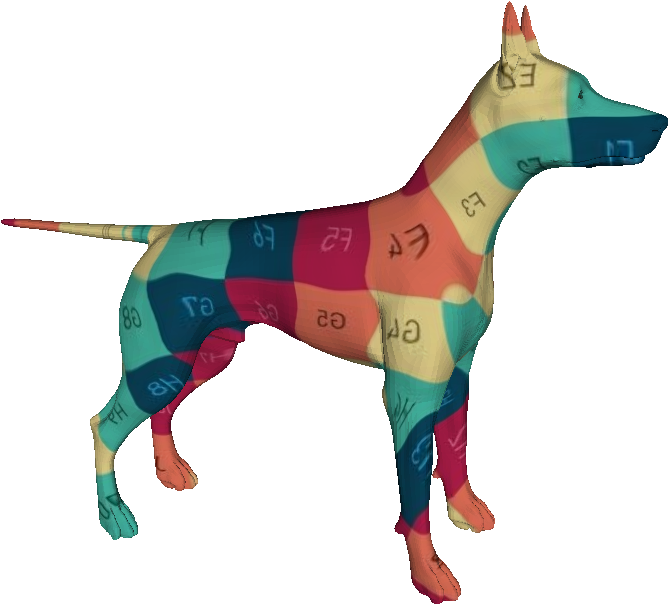}
	\caption{Example of a correspondence on the dog from the TOSCA data set calculated with our method. Same color means the points were matched with each other. }
	\label{fig:exampleTOSCA}
\end{figure}

\subsection{Registration}\label{subsec:expregistration}
We apply our framework to the FAUST dataset \cite{Bogo:CVPR:2014}, which contains data from scans of real humans with different poses. Each of these shapes has approximately $200k$ vertices and some of them are severely affected by topological noise. We set $\sigma^2:=0.01$, $K:=3000$ and we use a temporal discretization of $T=20$ steps. Again we match the null shape of every person to all its other poses. In Figure \ref{fig:faust} we display the surface distance of the morphed shapes to the goal shape for some examples.

\begin{figure*}
	\begin{overpic}
		[width=.19\linewidth]{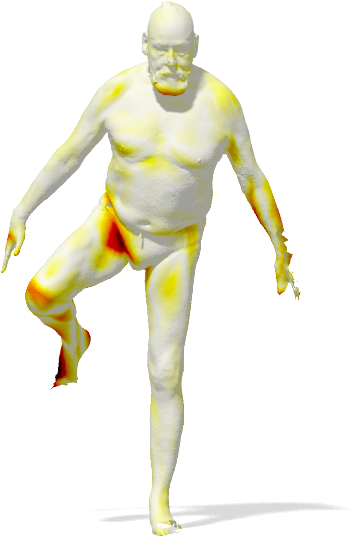}
		\put(0,-5){\tiny{avg=$0.67$cm\ \ max=$7.65$cm}}
	\end{overpic}
	\begin{overpic}
		[width=.13\linewidth]{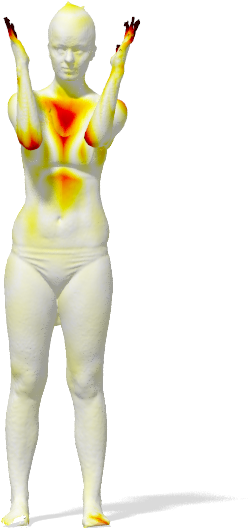}
		\put(-15,-5){\tiny{avg=$0.66$cm\ \ max=$10.12$cm}}
	\end{overpic}
	\begin{overpic}
		[width=.18\linewidth]{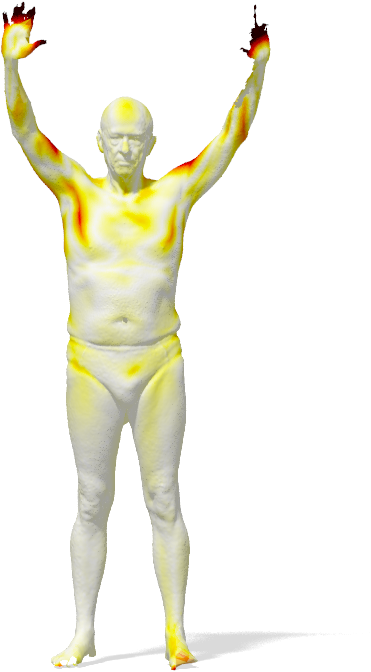}
		\put(0,-4){\tiny{avg=$0.63$cm\ \ max=$14.45$cm}}
	\end{overpic}
	\begin{overpic}
		[width=.17\linewidth]{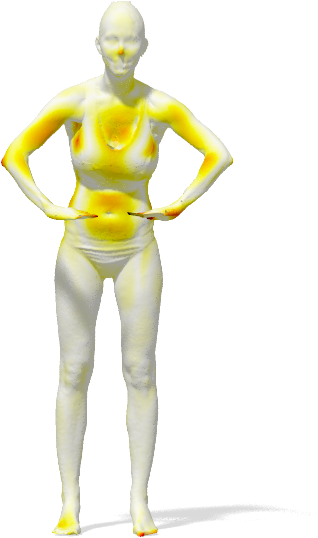}
		\put(0,-4.5){\tiny{avg=$0.65$cm\ \ max=$7.43$cm}}
	\end{overpic}
	\begin{overpic}
		[width=.14\linewidth]{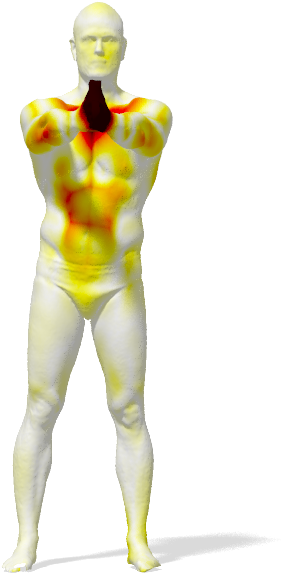}
		\put(-5,-4){\tiny{avg=$1.21$cm\ \ max=$15.23$cm}}
	\end{overpic}
	\begin{overpic}
		[width=.14\linewidth]{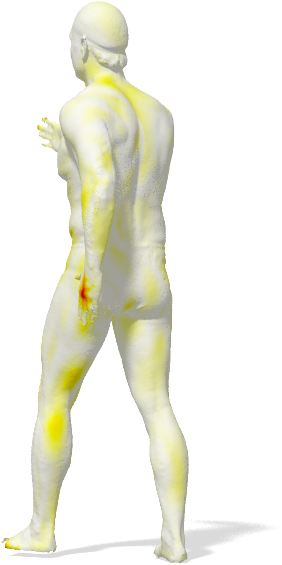}
		\put(-5,-4){\tiny{avg=$0.42$cm\ \ max=$3.65$cm}}
	\end{overpic}
	\begin{overpic}
		[angle=-90,width=0.01\linewidth]{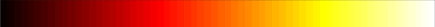}
		\put(0,-7){\small $0$}
		\put(0,102){\small $5$}
	\end{overpic}
	\caption{Example registrations from the FAUST scan data set. The surface color corresponds to the Euclidean surface distance between scan and registration. The scale is the same on all plots. All measures are in cm. We report the average and maximum error under each image. Many errors occur due to the SHOT descriptors being corrupted at holes and in noisy areas (e.g. the hands), the volume being exactly preserved although this is only an approximate property due to noise in real scans and topological changes (second to the right). }
	\label{fig:faust}
\end{figure*}

\subsection{Shape Interpolation}\label{subsec:expinterpolation}

\begin{figure}
	\centering
	\includegraphics[width=.95\linewidth]{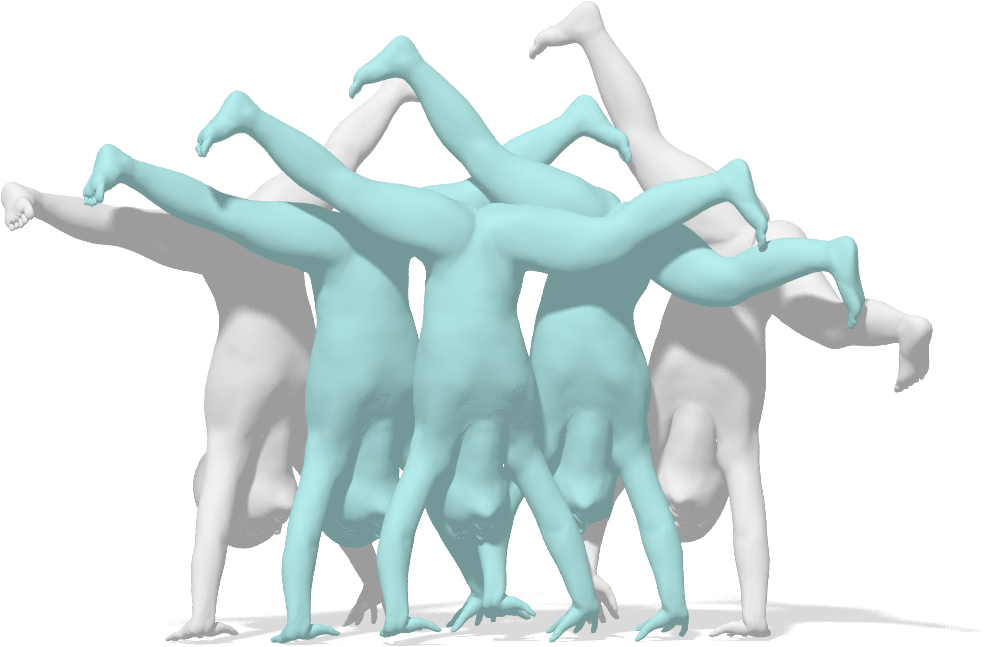}
	\caption{Example of an interpolation between two input shapes (white) from the Kids dataset \cite{rodola-cvpr14}. The interpolated shapes (blue) are at times $t=\{0.25, 0.5, 0.75\}$. }
	\label{fig:interp_kid}
\end{figure}

In comparison to other shape matching approaches, our setup models deformations in a comprehensive manner. It has a built in description of the actual transformation shapes undergo over time to morph into one another. Therefore it produces interpolated shapes a byproduct.

In some sense, our approach can be considered to be the extension of other shape registration methods based on a Gaussian mixture model representation of a point cloud \cite{myronenko2010pointsetregistration}, \cite{ma2014robust}. These particular methods correspond to our approach, if we choose to integrate the initial value problem \eqref{eq_initalvalueproblem} with the forward Euler scheme and $T=1$ time step. In particular this leads to a linear dependence of the quantity $f_n$ on the unknowns $a$ and the mapping $f$ for this case admits the following form:
\begin{equation}
\label{eq_oneeulerstep}
f_n:=f(x_n)=x_n+\sum_{k=1}^{K}v_k(x_n)a_k.
\end{equation}

As a consequence the shift \eqref{eq_oneeulerstep} of each point $x_n$ over time yields affine linear trajectories. Therefore the intermediate configurations $x^{\text{lin}}(t)$ are equivalent to a pointwise linear interpolation.
In contrast to this our method produces point trajectories which correspond to solutions $x(t)$ of the ODE \eqref{eq_initalvalueproblem}. We can now evaluate those at an intermediate time $t\in[0,1]$. Three examples of this qualitative evaluation were already displayed in Figure \ref{fig:interpol}, another one is provided in \ref{fig:interp_kid}.

In comparison to our approach linear interpolation distorts the shapes considerably, see Figure~\ref{fig:ipol_compare}. It also changes the size of certain parts of the shape which is inconsistent with transformations in the real world. In contrast to this our method is locally volume preserving. Nevertheless, it can still stretch and bend the shapes. Therefore, it is suitable for modeling shape morphing for rigid, as well as elastic objects. Linear interpolation performs especially bad if the considered transformations include large rotations.

\begin{figure}[b]
	\centering
	\includegraphics[width=0.45\linewidth]{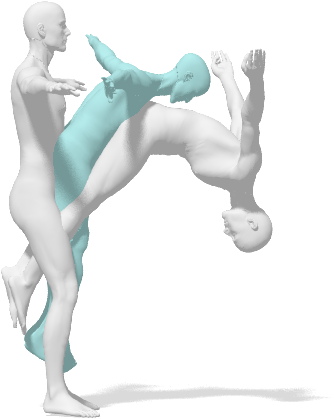}
	\includegraphics[width=0.45\linewidth]{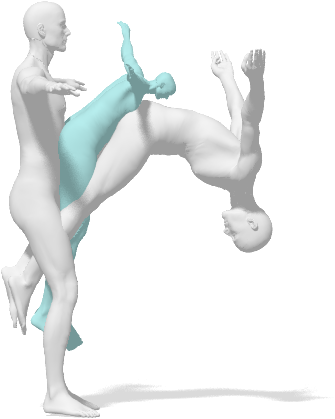}
	\caption{Comparison of the interpolated shapes at time $t=0.5$ produced by our method (left) and linear interpolation (right). This is example is hard because the head is rotated by more than $90$ degrees in the target. While linear interpolation squeezes the head in between our method finds a more realistic solution.
	}
	\label{fig:ipol_compare}
\end{figure}

\paragraph*{Extrapolation}

According to \eqref{eq_initalvalueproblem} the deformation field is independent of the time $t$ which makes the ODE autonomous. Therefore, we can also use the computed vector field $v$ to determine the solutions of this initial value problem at times $t>1$ which produces extrapolated shapes, see Figure~\ref{fig:extrapol_dog}. It is obviously an underdetermined task and it is hard to evaluate quantitatively. Nevertheless the extrapolation shapes our method produces are in many cases quite realistic for moderate time spans $t\in[1,1.5]$. Another example is displayed in Figure~\ref{fig:extrapol_handstand}. We observe that the speed of the extrapolated shapes seems to slow down after a certain timespan, especially when the shape is moving in previously unoccupied space. For the optimization there is simply no incentive to impose any particular movement on these parts of the domain $\Omega$. However, the volume preservation and uniformity assumption infer to the algorithm that an extension of the previous movement to some extend is desirable. Overall the resulting shapes are visually appealing and not too severely affected by distortions.

\begin{figure}[b]
	\centering
	\includegraphics[width=0.45\textwidth]{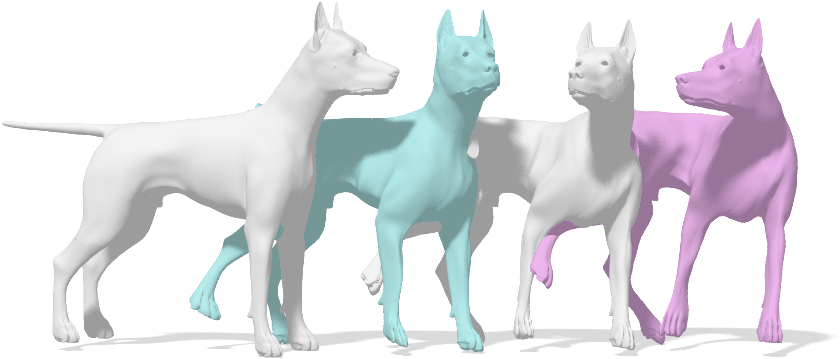}
	\caption{Example of an extrapolated shape produced by our method for two shapes. It can be determined using the temporally fixed deformation field $v$ for simulating the initial value problem \eqref{eq_initalvalueproblem} up to the time $t=1.3$. Source and target shape are white, one interpolated shape is shown in blue and the extrapolation is pink.}
	\label{fig:extrapol_dog}
\end{figure}

\begin{figure}[b]
	\centering
	\includegraphics[width=0.95\linewidth]{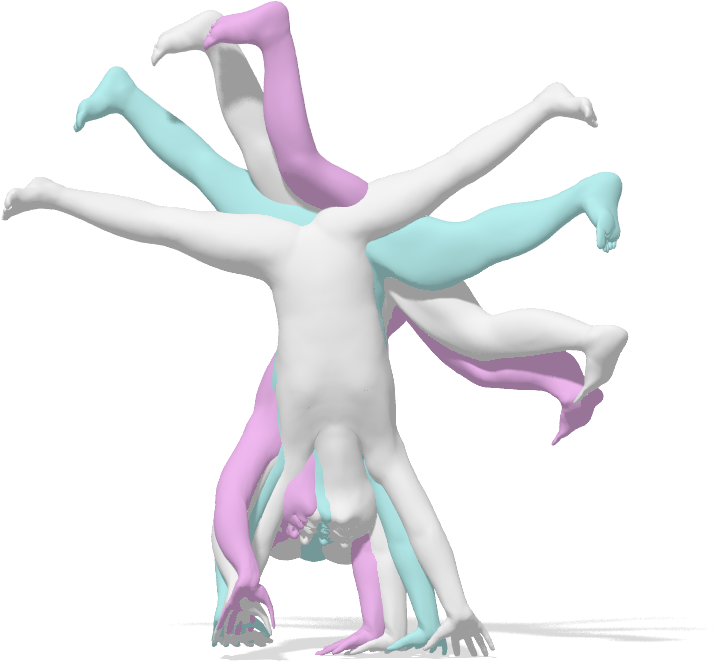}
	\caption{Example of an extrapolated shape from the KIDS dataset \cite{rodola-cvpr14} at the time $t=1.3$. The deformation field is usually magnified in the area between the input shapes and fans out in several directions further away from the input shapes. Therefore, choosing a really high time does not lead to broken shapes but the movement slows down more and more until it basically stops. }
	\label{fig:extrapol_handstand}
\end{figure}

\section{Conclusion}\label{sec:conclusion}

We presented a novel method solving the shape correspondence
problem while simultaneously computing a smooth,
volume preserving deformation field between the input
shapes. Furthermore, this deformation can be used to
efficiently calculate plausible interpolated shapes between the inputs at any intermediate point in time. The method consists of
two parts, the first is the optimization of the deformation
field using an expectation maximization approach and the
second applies the deformation to the input shapes using a Runge-Kutta scheme. The big advantage is that a subsampling
with around $3000$ vertices is sufficient to obtain the
deformation field defined in the continuous embedding space
due to our choice of basis. Therefore, the result can be
applied to any resolution mesh without slowing the
optimization.

We show quantitative results for shape correspondence and
registration that can compare to state-of-the-art methods for
these specific tasks and examples of shape interpolation and
extrapolation that arise naturally from our pipeline.

\begin{figure}[b]
	\centering
	\includegraphics[width=0.95\linewidth]{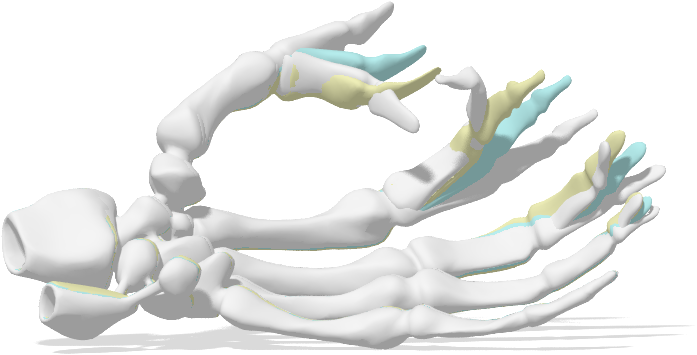}
	\caption{ Example of a failure case. The source and target shapes are white, the interpolated shape at $t = 0.5$ is blue and the resulting shape at $t = 1$ is yellow. The yellow is supposed to be as close to the target as possible but fails to do so in this case. The thumb and index finger are supposed to move in spatially very close areas (although at different time steps). Because we are only calculating one deformation field for all time steps our method does not end up in the target pose. Instead the finger are pushed away from each other. }
	\label{fig:hand_failure}
\end{figure}

\subsection{Limitations}\label{subsec:limitations}

Due to our choice of basis the deformation field is forced to be volume
preserving. This makes sense in applications with the same object but
prevents inter-object matchings - for example between two humans with
different body shapes.

For the same reason, our method has problems with topological changes. According to Proposition \ref{proposition:volumepreservation} the volume preservation property applies to every subregion of the domain $\Omega$, including the intermediate space between parts of the shape. Therefore, separating two touching parts (for example two hands) is in theory possible but requires
many high frequency deformation basis elements which makes the optimization costly.

Since there is not one unique, volume preserving deformation between two shapes,
our interpolation is not guaranteed to be as-rigid-as-possible which is plausible
in many applications. If the displacement is spatially far, we might end up with
squeezed intermediate states that are volume preserving but are affected by undesirable distortions.

The assumption of \eqref{eq_initalvalueproblem} being autonomous can also be problematic, if different parts of the shape move through the same region of the embedding space in a contradictory manner. One example for this is a hand closing to a fist. At first the index and middle finger occupy parts of the embedding space before the thumb moves in the same area but in a different direction. See Figure~\ref{fig:hand_failure}. A possible remedy for this problem is making the deformation fields time dependent.

\subsection{Future Work}\label{subsec:futurework}

Right now, the method will always find a solution that is globally volume
preserving. This allows to find good deformations fields in the case of
severe non-rigid deformations but is not applicable to partial data. In the
future, we want to extend this method to work on real scans, for example from
the Kinect, which naturally only show partial shapes by making the deformation
only locally volume preserving. This might also help with the separation of
close parts and handling non volume preserving deformations like style
or class changes.
Furthermore, we only calculate one time independent field for the entire deformation which means
mass at one spatial point always needs to move in the same direction, even at
a later time step. This restricts the complexity of the deformations that our method can handle, especially for large-scale motions over a longer period of time. It also leads to problems when non-matching
parts of the input shapes overlap in the initialization. Future versions should
allow more flexible types of deformation fields to extend it to a broader range of applications. We could for example associate different parts of the shape with different deformations fields or let them vary over time to address more difficult tasks.

{\small
	\bibliographystyle{eg-alpha-doi}
	\bibliography{refs}
}

\end{document}


\setcounter{equation}{22}

\maketitle

\section{Optimization: Details}\label{sec:optimizationdetails}

In section \ref{subsec:optimization} of the paper the expectation maximization framework of our method was already outlined. Here, we want to provide a more detailed description of the method. At the same time, we try to keep it as brief as possible, because most parts of this summary are standard techniques when dealing with Gaussian mixture models.

\subsection{Gaussian mixture model}\label{subsec:gmm}

As outlined above, we interpret the shifted points $f_n=x_n^{(T)}$ as the centers of Gaussian distributions with the covariance matrix $\sigma^2 I_D\in\mathbb{R}^{D\times D}$ which in the end should describe $\mathcal{Y}$ well. Furthermore, each point $y_m$ is assumed to correspond to some point $x_n$. This relationship is assumed to be encoded by the correspondence matrix $Z\in\{0,1\}^{(N+1)\times M}$, where $\sum_{n=1}^{N+1}Z_{nm}=1$. If $Z_{(N+1)m}=1$ for some $m$, the point $y_m$ does not correspond to any point $x_n$ and it is assumed to be uniformly sampled from $\Omega$ instead. This way we counteract the effect of outliers by acknowledging their presence and building it into our model.

According to Bayes' theorem the posterior probability distribution of the desired parameters $a_k$ in \eqref{eq_trunckarhunenloeve} given the latent correspondences $Z_{nm}$ and the observed points $\mathcal{Y}$ is defined as follows:
\begin{multline}
\label{eq_posteriorA}
p(a|Z,\mathcal{Y})\propto p(a)p(\mathcal{Y}|Z,a)=p(a)\prod_{m=1}^{M}p(y_m|Z,a)\\=p(a)\prod_{m=1}^{M}\prod_{n=1}^{N+1}p(y_m|Z_{nm}=1,a)^{Z_{nm}}.
\end{multline}

In the context of \eqref{eq_trunckarhunenloeve} it was mentioned that the prior of the parameters $a_k$ is a Gaussian distribution $a_k\sim\mathcal{N}(0,\lambda_k)$. As a shorthand notation we now define the diagonal matrix $L:=\mathrm{diag}(\lambda_1,...,\lambda_K)$ and set the prior $a\sim\mathcal{N}(0,L)$ for the coefficient vector $a$. In order to explicitly evaluate the posterior density of $a$ in \eqref{eq_posteriorA} we have to investigate the data likelihood $p(y_m|Z_{nm}=1,a)$ in detail. For $n<N+1$ it is now constructed to be the density of a Gaussian distribution in the product space of the embedding space $\Omega$ and the space of descriptor values:
\begin{equation}
\label{eq_probym}
p(y_m|Z_{nm}=1,a)=\frac{1}{(2\pi\sigma^2)^{\frac{D}{2}}}\exp\biggl(-\frac{1}{2\sigma^2}\mathrm{d}_{nm}^2\biggr).
\end{equation}

For the case $n=N+1$ it is simply the density of the uniform distribution, because $y_m$ is considered to be an outlier:
\begin{equation}
\label{eq_probymNplus1}
p(y_m|Z_{(N+1)m}=1,a)=1.
\end{equation}

Note that the GMM is not defined solely on the $D$ dimensional embedding space, but rather on the product space of $\Omega$ and the (possibly high dimensional) feature space. However, this only affects how the correspondences are computed and can be considered a theoretical nuance.

\subsection{Expectation maximization}\label{subsec:em}

We want to determine the coefficients $a$ by applying an expectation maximization approach similar to \cite{myronenko2010pointsetregistration}. In the E step soft correspondences $W\in[0,1]^{(N+1)\times M}$ are determined as a relaxed version of the latent variables $Z$:
\begin{multline}
W_{nm}=\mathbb{E}_{Z|\mathcal{Y},a}(Z_{nm})=p(Z_{nm}=1|y_m,a)\\=\frac{p(y_m|Z_{nm}=1,a)}{p(y_m|a)}=\frac{p(y_m|Z_{nm}=1,a)}{\sum_{\tilde{n}=1}^{N+1}p(y_m|Z_{\tilde{n}m}=1,a)}.
\end{multline}

The data likelihood terms $p(y_m|Z_{nm}=1,a)$ are defined in \eqref{eq_probym} and \eqref{eq_probymNplus1}. This leads to the expression proposed in \eqref{eq_weightsW}. Note, that the matrix $W$ has one more row than in the previous definition above. However, these two formulations are equivalent because this version of $W\in[0,1]^{(N+1)\times M}$ is assumed to be stochastic and the last row can be computed using the other entries. Moreover, the final algorithm does not explicitly depend on the entries of the last row of $W$ anyway and therefore the two formulations are interchangeable.

The M step now consists of minimizing the following energy with respect to $a$:
\begin{multline}
\label{eq_energyDist}
\mathbb{E}_{Z|\mathcal{Y},a}\bigl(-\log p(a|Z,\mathcal{Y})\bigr)\\=-\log p(a)-\sum_{m=1}^{M}\sum_{n=1}^{N+1}W_{nm}\log p(y_m|Z_{nm}=1,a)\\\propto\frac{1}{2}a^TL^{-1}a+\frac{1}{2\sigma^2}\sum_{m=1}^{M}\sum_{n=1}^{N}W_{nm}\|y_m-f_n\|_2^2.
\end{multline}

Note that the log likelihood of $y_m$ for the case $n=N+1$ is simply zero, see \eqref{eq_probymNplus1}. In this context the shifted points $f_n$ depend on the unknown coefficients $a$. However, the descriptor distances $\mathrm{d}^{\text{SHOT}}$ are independent of $a$, thats why the last proportionality in \eqref{eq_energyDist} holds. In particular this means that the last two expressions are equivalent up to summands that do not depend on $a$.

\subsection{Robust correspondences}

As proposed in \eqref{eq_energyE} we will reformulate the correspondence penalization term $\frac{1}{2}\|y_m-f_n\|_2^2$ in \eqref{eq_energyDist} to in order to make our method more robust:
\begin{equation}
\label{eq_corrpenelization}
\rho(\|y_m-f_n\|_2).
\end{equation}

Note that this heuristic admits a probabilistic interpretation, but we refrain from providing the background to this property. We choose the outer slope as $r_0:=0.01$. For values $\|y_m-f_n\|_2\leq r_0$ the term \eqref{eq_corrpenelization} remains exactly the same, but for bigger residua the penalization by the Huber loss only grows linearly. Due to this property the Huber norm does not penalize outliers exorbitantly high and is therefore more robust than the standard least squares loss.

\subsection{Algorithm}

We minimize the energy $E$ from \eqref{eq_energyE} with a Gauss-Newton type iteration scheme. A thorough derivation is provided in the next subsection. The complexity of this approach stems from the elaborate dependency \eqref{eq_rungekutta} of $f_n$ on the deformation parameters $a$. We already mentioned that we want to perform an expectation maximization algorithm which determines the coefficients $a$ and the weights $W$ in an alternating manner. In order to avoid a nesting of iteration schemes we refrain from computing exact minimizers of $E$ every time we update $a^{(i)}$. Instead we merely perform one Gauss-Newton update step per iteration.

\subsection{Gauss Newton}

According to \eqref{eq_trunckarhunenloeve} the deformation field $v$ is exactly determined by the coefficient vector $a\in\mathbb{R}^K$ for the representation we use. In order to compute the optimal deformation parameters $a$ we have to minimize the energy $E(a)$ defined in \eqref{eq_energyE}. For this purpose we will first discuss how to optimize the following energy, where the robust Huber loss was replaces by the standard least squares loss:
\begin{equation}
\label{eq_energyEnoHuber}
E^{\text{LS}}(a):=\frac{1}{2}a^TL^{-1}a+\frac{1}{2\sigma^2}\sum_{m=1}^{M}\sum_{n=1}^{N}W_{nm}\|y_m-f_n\|_2^2.
\end{equation}

We assume in this context that the correspondence $W$ are fixed. For this optimization we incorporate a Gauss-Newton type method which yields an iteration scheme converging to a local minimum of $E^{\text{LS}}$. Like in the Levenberg-Marquardt algorithm \cite{levenberg1944leastsquares} the iteration will contain an additional damping term $L^{-1}$ which is added to the Hessian of the non-linear least squares term.

Applying the standard Gauss-Newton methodology we get an iterative method to determine the weights $a$. The general idea of this approach is that the shifted points $f_n(a)$ are linearized around the current iterate $a^{(i)}$ in the energy \eqref{eq_energyEnoHuber}:
\begin{multline}
\label{eq_linEnergyE}
E^{\text{LS}}(a)\approx\\\frac{1}{2}a^TL^{-1}a+\frac{1}{2\sigma^2}\sum_{m=1}^{M}\sum_{n=1}^{N}W_{nm}\|y_m-\underbrace{\bigl(f_n(a^{(i)})+D_af_n(a^{(i)})(a-a^{(i)})\bigr)}_{\approx f_n(a)}\|_2^2.
\end{multline}

This approximate energy is linear in the current unknown $a$ so the remaining task is a simple linear least squares problem. The recursion formula to compute the approximate deformation parameters $a^{(i)}$ then admits the following explicit form:
\begin{equation}
\label{eq_coefficientsA}
a^{(i+1)}:=a^{(i)}-\bigl(J^T\tilde{W}J+\sigma^2L^{-1}\bigr)^{-1}\bigl(J^Tr-\sigma^2L^{-1}a^{(i)}\bigr).
\end{equation}

In this context $J$ consists of the Jacobians of $f_n$, $r$ of the (weighted) distance residuals and $\tilde{W}$ is a diagonal matrix containing the column sums of $W$. Let $e_D=(1,...,1)^T\in\mathbb{R}^D,e_M=(1,...,1)^T\in\mathbb{R}^M$, then these quantities are explicitly defined as:
\begin{subequations}
	\begin{equation}
	\label{eq_jacobianf}
	J=\begin{pmatrix}\mathrm{D}_af_1\\\vdots\\\mathrm{D}_af_N\end{pmatrix}\in\mathbb{R}^{ND\times K}.
	\end{equation}
	\begin{equation}
	r=\begin{pmatrix}\sum_{m=1}^{M}W_{1m}(f_1-y_m)\\\vdots\\\sum_{m=1}^{M}W_{Nm}(f_N-y_m)\end{pmatrix}\in\mathbb{R}^{ND}.
	\end{equation}
	\begin{equation}
	\tilde{W}=\textrm{diag}\bigl((We_M)\otimes e_D\bigr)=\textrm{diag}\begin{pmatrix}e_D\sum_{m=1}^{M}W_{1m}\\\vdots\\e_D\sum_{m=1}^{M}W_{Nm}\end{pmatrix}\in\mathbb{R}^{ND\times ND}.
	\end{equation}
\end{subequations}

What is left to specify is how to compute the derivatives $\mathrm{D}_af_n\in\mathbb{R}^{D\times K}$ in \eqref{eq_jacobianf}. Note that $f_n=x_n^{(T)}$ is recursively defined in \eqref{eq_rungekutta} and differentiating $f_n$ wrt. $a$ is not entirely trivial. However, it can be done in a straightforward manner by applying the chain rule to each element of the recursion. As a result the derivative $\mathrm{D}_ax_n^{(t)}$ is passed from the first time step $t=0$ to the last $t=T$ and gradually modified in each step. Inserting the Karhunen-Loève representation \eqref{eq_trunckarhunenloeve} in the definition \eqref{eq_rungekutta} yields the following recursive formula:
\begin{equation}
x_n^{(t+1)}(a):=x_n^{(t)}(a)+h\sum_{k=1}^{K}v_k\bigl(x_n^{(t)}(a)\bigr)a_k.
\end{equation}

The dependencies on $a$ are denoted explicitly in order to make it more comprehensible. The quantities $x_n^{(t)}(a)$ can now be differentiated wrt. $a$:
\begin{subequations}
	\begin{equation}
	\mathrm{D}_ax_n^{(0)}=0.
	\end{equation}
	\begin{multline}
	\mathrm{D}_ax_n^{(t+1)}=\biggl(I_D+h\sum_{k=1}^{K}\mathrm{D}_xv_k\bigl(x_n^{(t)}\bigr)a_k\biggr)\mathrm{D}_ax_n^{(t)}+\\h\begin{pmatrix}\vline&&\vline\\v_1\bigl(x_n^{(t)}\bigr)&\dots&v_K\bigl(x_n^{(t)}\bigr)\\\vline&&\vline\end{pmatrix}.
	\end{multline}
\end{subequations}

Note that the Jacobian $\mathrm{D}_xv_k\in\mathbb{R}^{D\times D}$ can be computed analytically for any basis element $v_k$. In this context $I_D\in\mathbb{R}^{D\times D}$ is the identity matrix.

The only thing left to discuss is how to extend this approach for the Huber loss penalization version of the energy $E$ in \eqref{eq_energyE}.
For point distances $\|y_m-f_n\|_2\leq r_0$ the Huber loss and the least squares loss are the same. For residual values $\|y_m-f_n\|_2>r_0$ the derivate wrt. the deformation parameters $a$ is the following:
\begin{equation}
\mathrm{D}_a\rho(\|y_m-f_n\|_2)=r_0\frac{(f_n-y_m)^T}{\|f_n-y_m\|_2}\mathrm{D}_af_n.
\end{equation}

This eliminates the possibility of a direct Gauss-Newton type optimization which requires non linear least squares terms.	We can however incorporate this in our algorithm using a simple heuristic. For this purpose we multiply the respective weights $W_{nm}$ with the factor $r_0\frac{1}{\|f_n-y_m\|_2}$, if $\|y_m-f_n\|_2>r_0$.

\section{Karhunen-Loève expansion of the deformation field}
\label{sec:minimalDirichlet}

We will now provide some theoretical justification for the particular choice of basis in \eqref{eq_basisV} and the construction of the weights \eqref{eq_lambdak}.
These $\lambda_k$ can be interpreted to be the eigenvalues of the linear operator $\mathcal{C}:=(-\Delta)^{-\frac{D}{2}}$ corresponding to the eigenfunctions $\phi_k$. We can then apply the so-called Karhunen-Loève expansion \cite[Ch. 11]{sullivan2015UC} to our setup. This framework provides us with an alternative representation of the potential field $\Phi$, which can in turn be used to define the deformation field $v$. For further reference concerning the mathematical foundation of this approach the interested reader is referred to \cite{stuart2010inverseproblems}, \cite{cotter2013MCMC}, \cite{dashti2017bayesian}. Following this approach one can now derive a construction which enables us to sample arbitrary square integrable scalar fields $\hat{\Phi}:\Omega\to\mathbb{R}$:
\begin{equation}
\label{eq_karhunenloeve}
\hat{\Phi}(x)=\sum_{k=1}^{\infty}\phi_k(x)\sqrt{\lambda_k}\xi_k.
\end{equation}

According to the Karhunen-Loève expansion the coefficients $\xi_k\sim\mathcal{N}(0,1)$ are samples of the standard normal distribution. This approach can now be applied to get an alternative description of each entry of the potential vector field $\Phi$. Inserting this in \eqref{eq_vectorfieldfrompotential} we obtain an alternative representation of the deformation field $v$. In particular we get the summation \eqref{eq_trunckarhunenloeve} for the basis elements \eqref{eq_basisV} in the 3D case. Indeed we can derive the following Gaussian prior distribution for the weights $a_k$:

\begin{equation}
a_k=\sqrt{\lambda_k}\xi_k\sim~\mathcal{N}(0,\lambda_k).
\end{equation}

The choice of the exponent $\frac{D}{2}$ in the definition of the weights $\lambda_k$ \eqref{eq_lambdak} is not arbitrary. In general it is supposed to be chosen strictly larger than $\frac{D}{2}$ in order for our resulting basis to fulfill certain approximation properties in the limit of infinitely many basis functions, see \cite[Ch. 2.4]{dashti2017bayesian}. However, we achieved good results in our experiments by choosing it as small as possible in order to not suppress the high frequencies more severely than necessary. In particular the expressiveness of or method seems to deteriorate when a large exponent is chosen, because then the weights \eqref{eq_lambdak} decay too rapidly. Therefore we typically even set it to $\frac{D}{2}$ which works fine for our purposes, although this is not theoretically justified when the number of basis functions approaches infinity. On the other hand, choosing it smaller than $\frac{D}{2}$ certainly causes the expected value of the velocity series to diverge for $K\to\infty$.

To conclude this section we want to motivate our choice of the Karhunen-Loève framework and the particular linear operator $\mathcal{C}$ to model the deformation fields $v$. In the context of the Karhunen-Loève expansion the operator $\mathcal{C}$ is called covariance operator. It is typically chosen to incorporate some assumptions about the regularity of the produced sample functions.
A natural assumption about the deformation fields $v$ is that they are as uniform as possible. This yields that the resulting correspondence mappings are to some degree spatially continuous. Therefore we require the Dirichlet energy to be small:
\begin{equation}
\label{eq_nablav}
\|\nabla v\|^2_{L_2}=\sum_{d=1}^{D}\int_\Omega\|\nabla v_d(x)\|_2^2\mathrm{d}x.
\end{equation}

We can achieve this by penalizing the high frequency components of $v$. These frequencies are strongly related to those of the potential field $\Phi$, because according to \eqref{eq_vectorfieldfrompotential} the basis elements are simply mapped onto the velocity basis elements. This mapping does not change the frequencies:
\begin{equation}
\|\nabla v\|_{L_2}=\|\nabla(\nabla\times\Phi)\|_{L_2}=\|\nabla\Phi\|_{L_2}.
\end{equation}

If we choose e.g. $D=2$ one can prove that the Dirichlet energy $\|\nabla v\|^2_{L_2}$ is equivalent to the squared $\ell_2$ norm of the weights $\xi$:
\begin{equation}
\|\xi\|_{\ell_2}^2=\|\nabla v\|^2_{L_2}\text{, for }D=2.
\end{equation}

A derivation of this property can be found in \cite[Ch. 7.1.3]{dashti2017bayesian}. In the case of finitely many parameters $\xi_1,...,\xi_K$ the norm $\|\xi\|_{\ell_2}$ is equivalent to the Euclidean norm $\|\xi\|_2$ of the vector $\xi=(\xi_1,...,\xi_K)^T$. The term $\|\xi\|_2^2$ is in turn proportional to the negative log likelihood of the standard normal distributed parameter $\xi\sim\mathcal{N}(0,I_K)$:

\begin{equation}
-\log(p(\xi))= \frac{K}{2}\log(2\pi)+\frac{1}{2}\|\xi\|_2^2\propto\frac{1}{2}\|\xi\|_2^2.
\end{equation}

This indicates that a maximum likelihood approach involving $\xi$ leads to an enforcement of uniformity of the vector field $v$. This can be extended to the case $D=3$ in a similar manner, but we refrain from providing more details here for the sake of brevity.

{\small
	\bibliographystyle{eg-alpha-doi}
	\bibliography{refs}
}